\documentclass[acmsmall,screen,nonacm]{acmart}

\AtBeginDocument{%
  \providecommand\BibTeX{{%
    \normalfont B\kern-0.5em{\scshape i\kern-0.25em b}\kern-0.8em\TeX}}}

\setcopyright{none}
\renewcommand\footnotetextcopyrightpermission[1]{}
\usepackage[utf8]{inputenc}
\usepackage[T1]{fontenc}

\usepackage{xcolor}
\usepackage[export]{adjustbox}

\usepackage[many]{tcolorbox}
\usepackage{caption}
\usepackage{graphicx}

\usepackage{stackengine}
\def\delequal{\mathrel{\ensurestackMath{\stackon[1pt]{=}{\scriptstyle\Delta}}}}

\definecolor{captionbgcolor}{RGB}{103,143,150}

\DeclareCaptionFormat{tcbcaption}{%
  \begin{tcolorbox}[
    nobeforeafter,
    colback=captionbgcolor,
    arc=0pt,
    outer arc=0pt,
    boxrule=0pt,
    colupper=white,
   fontupper=\normalsize\sffamily,
    boxsep=0pt
  ]
  #1#2#3
  \end{tcolorbox}%
}
\newtheorem{assump}{Assumption}
\newtheorem{definition}{Definition}
\newtheorem{proposition}{Proposition}

\usepackage[linesnumbered, ruled, vlined]{algorithm2e}
\usepackage{mathtools}
\usepackage{threeparttable}
\usepackage[noend]{algpseudocode}
\usepackage{indentfirst}
\usepackage{graphicx}
\usepackage{subfigure}
\usepackage{amsmath}
\usepackage{breqn}
\makeatletter
\newcommand{\rmnum}[1]{\romannumeral #1}
\newcommand{\Rmnum}[1]{\expandafter\@slowromancap\romannumeral #1@}
\makeatother
\usepackage{array}
\newcolumntype{C}[1]{>{\centering\arraybackslash}m{#1}}

\SetKwInOut{Begin}{Begin}
\SetKwInOut{Input}{Input}
\SetKwInOut{Output}{Output}
\SetKw{Continue}{continue}
\SetKw{Break}{break}

\usepackage{cleveref}
\usepackage{caption}

\title{MarS-FL: Enabling Competitors to Collaborate in Federated Learning}

\author{Xiaohu Wu}
\affiliation{%
  \institution{Nanyang Technological University}
  \country{Singapore}
}
\email{xiaohu.wu@ntu.edu.sg}

\author{Han Yu}
\affiliation{%
  \institution{Nanyang Technological University}
  \country{Singapore}
}
\email{han.yu@ntu.edu.sg}

\begin{abstract}
Federated learning (FL) enables multiple data owners to collaboratively participate in the machine learning (ML) training process without data sharing, which is crucial to the application of ML to various scenarios as data privacy is a growing concern of modern societies.
Among the three use cases of FL is an unaddressed one where the FL participants (FL-PTs) rely on ML models to provide services in a competitive market, where no FL-PTs can accept a significant reduction in their market shares that represent their competitiveness. Currently, there is no means for identifying the market conditions under which the application of FL in a specific market is viable, and specifying the requirements that have to be imposed while applying FL under these conditions. In this paper, we propose the \underline{mar}ket \underline{s}hare-based decision support framework for participation in \underline{FL} (MarS-FL).
Specifically, we introduce {\em two notions of $\delta$-stable market} and {\em friendliness} to measure the viability of FL and the market acceptability of FL. The FL-PT behaviours can be predicted using game theoretic tools (i.e., their optimal strategies concerning participation in FL). If the market $\delta$-stability is achievable, the final model performance improvement of each FL-PT shall be bounded, which relates to the market conditions of FL applications. We provide tight bounds and quantify the friendliness, $\kappa$, of given market conditions to FL. Experimental results show the viability of FL in a wide range of market conditions. This paper provides a conceptual framework that serves as a starting point and allows for the further development of FL in the scenarios where FL-PTs are competitors.
\end{abstract}
\begin{document}

\flushbottom
\maketitle
%
%


\section{Introduction}
\label{sec.introduction}

Federated learning (FL) is an emerging privacy-preserving collaborative machine learning (ML) paradigm and has gained widespread attention \cite{Kairouz-et-al:2021}. An individual data owner ({\em a.k.a.} FL participants (FL-PTs)) often has insufficient data to train its learning model.
Multiple FL-PTs can collaboratively build better ML models with their local data while no local data of an individual are exposed and transferred outside, thereby preserving data privacy by design \cite{Yang19a,cheng2020federated}. {For example, in a centralized FL architecture, the FL-PTs periodically uploads their respective local ML model updates to a central server that will aggregate these updates for collaboratively training a global ML model, instead of uploading local data. This process is illustrated in Figures~\ref{Fig-FL}\subref{Fig-Federated-Learning} and iterated until the global ML model converges.}



As data privacy has become a growing concern in modern societies and world-wide governments are enacting related laws \cite{cheng2020federated}, FL has become especially important and it has many promising applications such as digital banking \cite{Long-et-al:2020,yang2021toward}, recommender systems for online services \cite{Tan2020federated,kalloori2021horizontal}, safety monitoring \cite{Liu20a}, healthcare \cite{Sheller-et-al:2020,Kaissis-et-al:2020} and mobile applications \cite{bonawitz2019towards}. With the development of 5G/6G networks, the applications of FL are also expanding rapidly \cite{letaief2019roadmap,yang2021federated}. In such context, proper incentive mechanisms are urgently needed to incentivize data owners to join the FL training process \cite{Zhan21a,Lyu20b,Kairouz-et-al:2021}. Based on the relationship of the eventual model users and FL-PTs, the prevailing use cases of FL can be categorized as follows:
\begin{enumerate}
    \item The FL-PTs are not the model users. The model users benefit from the model performance improvement, which has no direct impact on FL-PTs' wellbeing. In this case, since FL-PTs are not interested in the model performance, monetary rewards are often used to incentivize FL-PTs to contribute more local resources to the FL training process \cite{Song19a,Sarikaya19a,Zhang21a,zhan2020learning,Yu20a,Yu-et-al:2020,Zeng20a}. This is illustrated in Figures~\ref{Fig-FL}\subref{Fig-FL-Indifferent-1-v5}.

    \item The FL-PTs are the eventual model users, and an FL-PT's utility is not affected by the model performance improvements obtained by other FL-PTs. {This is illustrated in Figures~\ref{Fig-FL}\subref{Fig-FL-Indifferent-2-v5} where each person only cares about its own health and ML model accuracy.} In this case, different FL-PTs achieve different levels of model performance improvements and monetary transfer is taken to let the FL-PTs of higher payoff compensate the FL-PTs of lower payoff \cite{Tang21a}.

   \item The FL-PTs are the eventual model users. However, an FL-PT's utility is dependent on the model performance improvements of all FL-PTs. {This is illustrated in Figures~\ref{Fig-FL}\subref{Fig-FL-Compete-v2} where each customer compares different FL-PTs (e.g., apps) and weighs the ML model performances of all FL-PTs before choosing an app for use.}
\end{enumerate}

\begin{figure}[t!]
\centering
\subfigure[A centralized FL architecture]{\includegraphics[width=0.8\columnwidth]{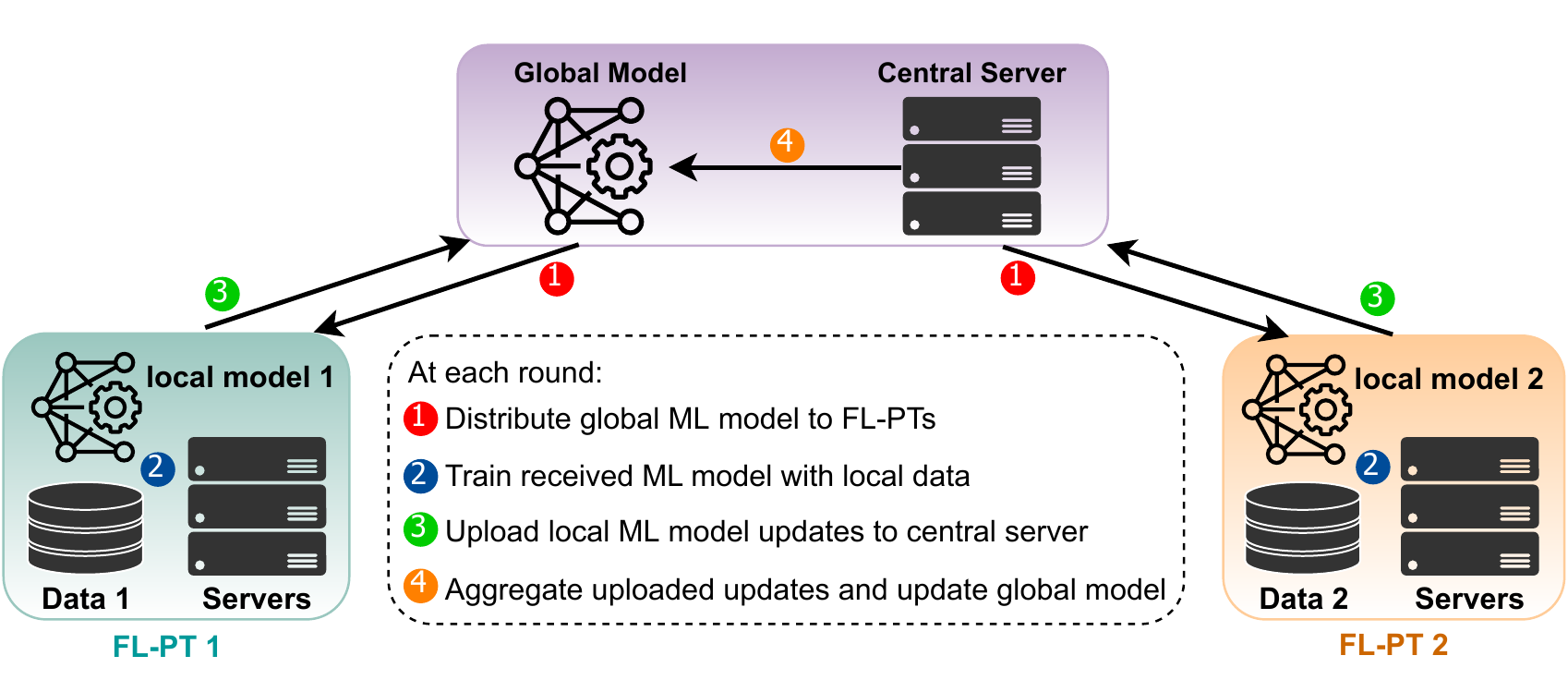}
\label{Fig-Federated-Learning}}\\
\subfigure[Case 1]{\includegraphics[width=0.09\columnwidth]{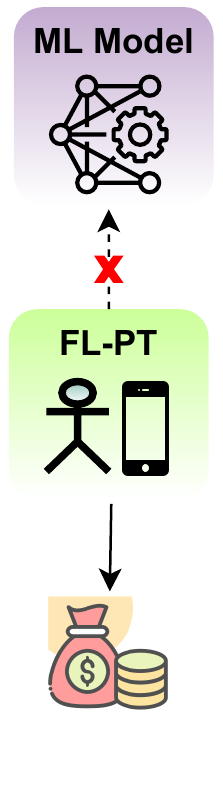}\label{Fig-FL-Indifferent-1-v5}}
\hspace{1.05cm}\subfigure[Use case 2]{\includegraphics[width=0.274\columnwidth]{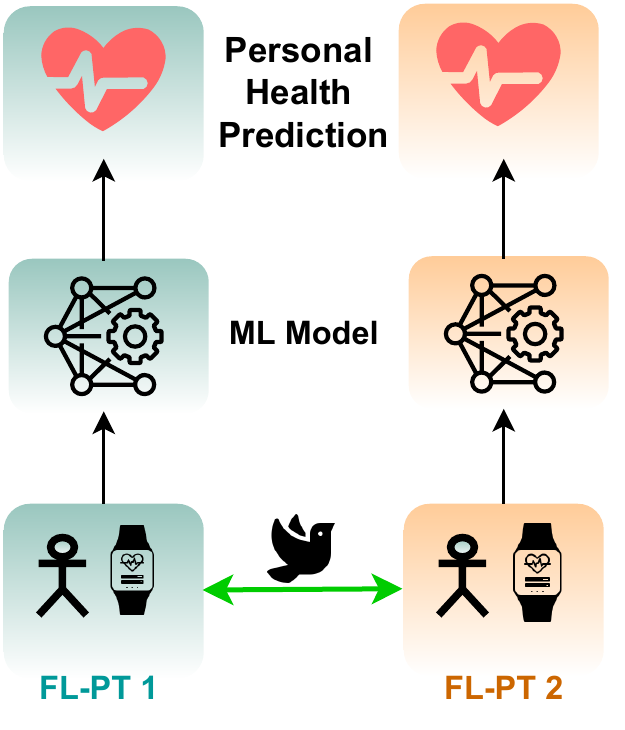}\label{Fig-FL-Indifferent-2-v5}}
\hspace{0.7cm}\subfigure[Use case 3: the use case of this paper]{\includegraphics[width=0.436\columnwidth]{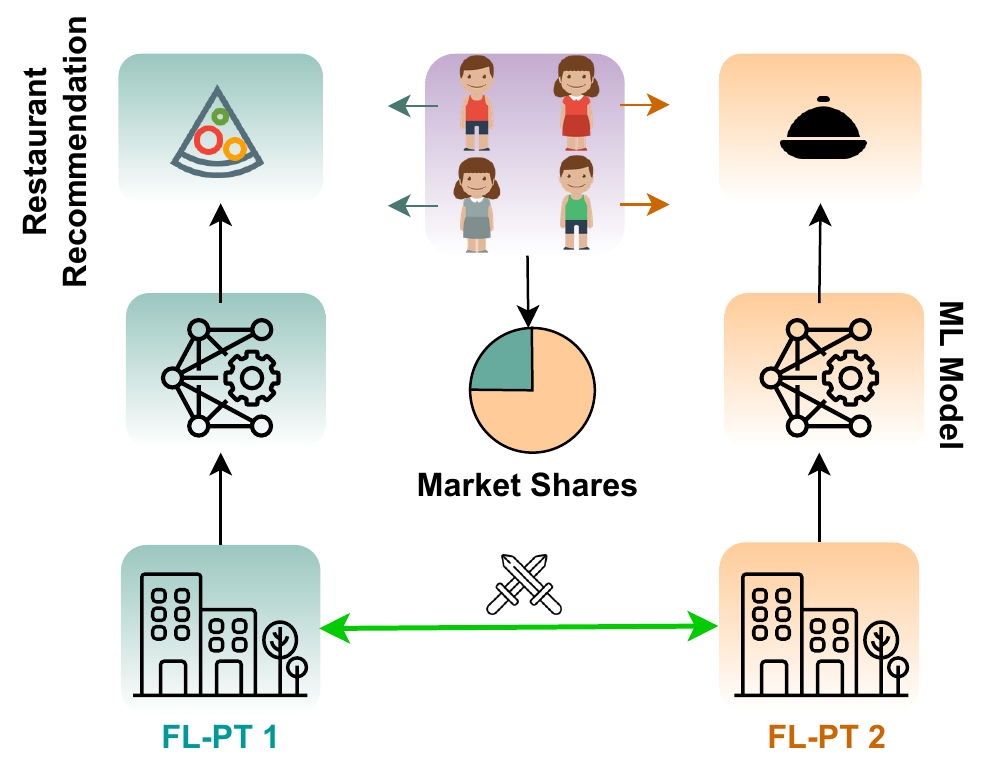}
\label{Fig-FL-Compete-v2}}
\caption{FL overview and its three use cases}
\label{Fig-FL}
\end{figure}

Although the third use case has been well recognized \cite{Zhan21a,Lyu20a,yang2021toward,kalloori2021horizontal}, its formal study is missing in existing literature, which hinders FL-based collaboration in practical situations. This is probably due to the unclearness of the metrics and additional constraints for judging and maintaining the viability of FL in a competitive market. In contrast, in the first and second use cases, such questions do not exist since the FL viability is obvious. In this paper, we study the third use case, which is commonly found in horizontal federated learning (HFL) scenarios in which the FL model is used for profit-making activities ({\em e.g.,} personalized loan interest determination in digital banking) \cite{Yang19a}. In such scenarios, FL-PTs provide the same services and compete for a same group of customers. 
Market share is a key indicator of FL-PTs' market competitiveness ({\em i.e.,} how well a firm is doing against its competitors) \cite{Zhan21a,Farris10a}. For example\cite{Zhan21a,Lyu20a}, when multiple digital banks collaboratively train a FL model to predict the creditworthiness of customers, larger banks with more high quality data may be reluctant to join FL for fear of benefiting its smaller competitors and eroding its market share. The underlying challenge has been framed as the ``free-rider problem'' in FL \cite{Kairouz-et-al:2021} for which monetary incentives are not effective. 

This challenge has inspired trust-based FL ecosystems to emerge, which are built on the proposition that the FL model performance improvement obtained by each FL-PT shall be proportional to its contribution \cite{Kairouz-et-al:2021, Lyu20b}.
However, there is no analytical tool to help FL-PTs determine how it affects their market shares under different market conditions. 
Thus, the data owners still face uncertainty regarding potential market share erosion as a result of joining FL, which hinders the adoption of FL in HFL scenarios involving competitors. As FL can improve FL-PTs' model performance which, in turn, results in better products or services to enhance the collective utility of the businesses and customers involved \cite{Yang19a,Kairouz-et-al:2021}, a clear understanding of the impact of joining FL on each FL-PT's market share to facilitate wider adoption of FL is crucial.

Suppose that the original market size is $P$ and there are $n$ firms ({\em i.e.,} FL-PTs) $\mathcal{C}=\{1, 2, \cdots, n\}$ in a given market which can join FL. The original market share of FL-PT $i\in\mathcal{C}$ is $MS_{i}\in (0,1)$, where $\sum\nolimits_{i=1}^{n}{MS_{i}}=1$. After joining FL, the market size becomes $P^{\prime}$ and the market share of FL-PT $i\in\mathcal{C}$ becomes $MS_{i}^{\prime}\in (0,1)$, where $\sum\nolimits_{i=1}^{n}{MS_{i}^{\prime}}=1$. We now define a notion of market stability to measure the viability of FL. Intuitively, no FL-PT is willing to accept a significant reduction in its market share as a result of joining FL (i.e., no FL-PT's interest shall be sacrificed), even if all FL-PTs benefit from FL with improved model performance and quality of service. Formally, we have:

\begin{definition}[\textbf{$\delta$-Stable Market}]\label{def-market-stability}
In a given FL ecosystem, the market is $\delta$-stable if the market share of every FL-PT $i\in\mathcal{C}$ satisfies the following condition: $V_{i} = MS_{i} - MS_{i}^{\prime} \leqslant \delta$,
where $\delta \in (-1, 1)$ is an upper bound of the difference of market shares for any given FL-PT before and after joining FL.
\end{definition}
We refer to $V_{i}$ as {\em the market variance of FL-PT $i$} as a result of joining FL. Ideally, the market share of each FL-PT $i$ does not vary (i.e., $\delta=0$). However, this is hard to achieve in practice. Thus, we relax it to $\delta$ not exceeding a small positive real number within $(0, 1)$ and that each FL-PT $i$ experiencing an acceptable reduction in its market share after joining FL in the worst case. Then, the FL ecosystem is said to be viable only if the market is $\delta$-stable.

In this paper, we propose an analytical framework for \underline{mar}ket \underline{s}hare-based decision support about participation in \underline{FL} (MarS-FL) by leveraging game theory and marketing models \cite{Osborne04a,Rust93a}. 
Given the original FL-PTs' market shares before joining FL, the market dynamics are driven by exogenous factors such as the customer loyalty, and the leaving and switching rates of customers for each FL-PT; {also see Figure~\ref{market-status}}. It can be embodied by a general mathematical model in economic literature \cite{Rust93a}. We also take an additional factor into account: the possible market growth rate. 
In the FL context, the switching customer of a FL-PT will reconsider choosing one of the $n$ FL-PTs. Due to  market growth, a customer who newly enters the market will choose one of the $n$ FL-PTs. In both cases, the possibility of choosing an FL-PT is positively correlated to the service quality of the FL-PT which, in turn, depends on the model performance of this FL-PT. Intuitively, the model performance improvement of each FL-PT is bounded, resulting in bounded variations in market shares.

For the decision-makers of an FL ecosystem, MarS-FL provides a tight {and elegant} lower bound of the minimum relative ML model performance improvement required by each FL-PT in order to maintain the market $\delta$-stability. These bounds can guide the allocation of the ML model performance improvements among FL-PTs so that they can be well motivated to join FL.
Based on these bounds, we further define a notion of {\em friendliness} to measure how conducive a market is for FL. It provides a sufficient and necessary condition for the viability of FL in a competitive market.
Through numerical experiments, we demonstrate the capability of MarS-FL to analyze the viability of FL under a wide range of market conditions. {Our results are useful for identifying the market conditions under which collaborative FL model training is viable among competitors, and the requirements that have to be imposed while applying FL under these conditions.}

\section{Related Work}

Our work in this paper is broadly related to incentive mechanism design for motivating participation in FL.
Existing research in this domain can be broadly divided into two categories \cite{Kairouz-et-al:2021}: 1) monetary FL incentive schemes, and 2) non-monetary FL incentive schemes.

Monetary incentive schemes are generally designed for the first FL use case. These schemes that provide monetary rewards to FL-PTs in order to motivate them to contribute more local resources to FL model training. FL-PTs in this case are often resource-constrained devices such as mobile devices. The types of resources include not only local data but also computational and communication resources. Sarikaya {\em et al.} \cite{Sarikaya19a} models the interaction between the FL server ({\em i.e.,} the model user) and FL-PTs as a Stackelberg game with the goal of improving the FL model performance by jointly optimizing the commitment of local computational resources and the allocation of the incentive budget. Zhan {\em et al.} \cite{zhan2020learning} also leverages the Stackelberg game to model the interaction to incentivize FL-PTs to contribute more data. Zeng {\em et al.} \cite{Zeng20a} considers motivating FL-PTs to contribute multiple types of resources by applying multi-dimensional procurement auction theory and proposed a scheme that selects and rewards a fixed number of FL-PTs for FL training.
Song {\em et al.} \cite{Song19a} and Zhang {\em et al.} \cite{Zhang21a} proposed schemes to select and pay FL-PTs based on reputation and reverse auction.
The FLI approach \cite{Yu20a,Yu-et-al:2020} has been proposed which supports fair allocation of incentive payout to FL-PTs using future revenues generated by the FL model. These approaches leverage incentive mechanism research results in economics and game theory. As they generally assume that the FL-PTs care only about monetary rewards, the intermediate FL models and the final FL model are generally freely shared during the FL training process. For the second use case, Tang {\em et al.} \cite{Tang21a} assume that each FL-PT will commit all its data for local model training, and characterize the interaction among FL-PTs as a non-cooperative game. They propose, under mild assumptions, an incentive scheme that achieves social welfare maximization, individual rationality and budget balance.

In the third FL use case in which the FL-PTs are also the end users of the final FL models and competing in the same market, which is the focus of our study, sharing the same FL model among all participants without regards to their contributions has been shown to cause breakdown of collaboration \cite{Lyu20a}. Non-monetary incentive schemes which assign each FL-PT a different model in each training iteration with performance reflecting its contribution are starting to emerge \cite{xu2020towards,Lyu20a,Lyu20b}. However, these existing approaches do not take FL-PTs' market shares into account when allocating different versions of FL models to them. MarS-FL bridges this gap by identifying the optimal FL participation strategies, the viable operational space of an FL system, and the market conditions under which FL can be beneficial for a given FL-PT.

\section{FL Market Dynamics}
In this section, we provide a detailed analysis of the market dynamics surrounding the third FL use case scenario. We first introduce two typical FL system architectures.

\subsection{FL System Architectures}

In FL, each FL-PT $i\in \{1, 2, \cdots, n\}$ has a local model with a common structure. The training process iterates for multiple rounds.
We use $w_{i}^{t}$ to denote the parameters of the local model of FL-PT $i$ at round $t\in\{1,2,\cdots,T\}$. In the {\em centralized FL architecture}, there is a central FL server to mediate the training process \cite{Tang21a}. In the beginning of round $t$, let $w^{t}$ denote the parameters of the global FL model owned by the FL server and $\hat{w}_{i}^{t}$ denote the model parameters sent from the FL server to FL-PT $i$. In the current design, $\hat{w}_{i}^{t}=w^{t}$, which is not applicable to the use case of this paper; however, the related questions can be addressed in future based on the results of this paper. Each FL-PT $i$ downloads $\hat{w}_{i}^{t}$. $i$ uses a batch of its local data to train model $\hat{w}_{i}^{t}$ and computes the gradient $\nabla w_{i}^{t}$. The updated local model is denoted as $w_{i}^{t}=\hat{w}_{i}^{t}-\eta \nabla w_{i}^{t}$ where $\eta$ is the learning rate. $i$ then uploads $w_{i}^{t}$ (or $\nabla w_{i}^{t}$) to the FL server. The FL server produces the global model $w^{t+1}$ by aggregating the received local model updates from all FL-PTs. Let $x_{i}$ denote the amount of data that $i$ decides to use for its local model training. In the classic Federated Averaging strategy, the aggregating process is defined as\cite{McMahan17a,Li2020On}: $$w^{t+1} = \frac{x_{i}}{\sum_{j=1}^{n}{x_{j}}}w_{i}^{t}.$$
In the {\em decentralized FL architecture}, the training process iterates without a central FL server \cite{Kairouz-et-al:2021}. In the beginning of round $t$, each FL-PT $i$ uses a batch of its local data to train its local model $w_{i}^{t-1}$ obtained at the last round and computes the gradient $\nabla w_{i}^{t}$. The updated model is denoted as $w_{i}^{t}=w_{i}^{t-1}-\eta \nabla w_{i}^{t}$. $i$ then shares the updates $\nabla w_{i}^{t}$ with other FL-PTs which have agreed to establish collaborative training partnership with $i$ (based on considerations such as trust \cite{Lyu20a}). In the meantime, $i$ also downloads local model updates from other partners. Then, $i$ further updates its model $w_{i}^{t}$ by aggregating the received local model updates from its partners.



\subsection{FL Market Model}
\label{sec.overview-problem}

Let us consider $n$ firms ({\em i.e.,} FL-PTs) $\mathcal{C}=\{1, 2, \cdots, n\}$ in a given market which can join FL. The market size of FL-PT $i$ is $P_{i}$ ({\em e.g.,} in terms of the number of customers) and their aggregated market size is denoted as $P$. The (relative) market share of $i$ is $MS_{i}\in (0,1]$ where
\begin{align}\label{equa-customer-number-i}
P_{i}=MS_{i}\times P.
\end{align}
Thus, we have $\sum_{i=1}^{n}{MS_{i}}=1$.
Firm $i$ owns a local dataset $\mathcal{D}^{i}$ with $D^{i}=|\mathcal{D}^{i}|$ samples. The local model performance can be measured by the model loss value, which in turn, affects the quality of service offered by the firm. Smaller loss values imply better quality of service.The variable $x_{i}\in \left[0, D^{i} \right]$ denotes the amount of local data that FL-PT $i$ decides to use for FL training.
If $i$ trains its model solely using its local data $\mathcal{D}^{i}$ without joining FL, the final loss function $L_{i}$ of the resulting model can be expressed as:
\begin{align*}
L_{i} = L_{i}\left( D^{i} \right).
\end{align*}
If $i$ joins FL, it can leverage local models from other FL-PTs. After the FL model training process ends, the model loss function $L_{i}^{\prime}$ of FL-PT $i$ can be expressed as:
\begin{align*}
L_{i}^{\prime} = L_{i}^{\prime}\left( \left\{D^{j}\right\}_{j=1}^{n},  \left\{x_{j}\right\}_{j=1}^{n}, Trad \right).
\end{align*}
$Trad$ denotes a given model exchange scheme among FL-PTs ({\em e.g.,} under the centralized FL architecture \cite{McMahan17a}, or the decentralized FL architecture \cite{Lyu20a}). It determines how much information $i$ obtains during the FL process and thus the value of $L_{i}^{\prime}$.

In this paper, we consider FL ecosystems with non-monetary incentive schemes in which an FL-PT's reward is reflected by the final loss function value of the model it obtains. We make a natural assumption that for an FL-PT $i\in \mathcal{C}$, its contribution to the FL system increases as it commits more of its local data for FL training, and the loss function value of $i$ decreases as its contribution to FL increases. Formally, the assumption is expressed as follows.

\begin{assump}\label{assump-more-data-contribution}
Suppose $D^{i} \geqslant x_{i}^{(1)}>x_{i}^{(2)}\geqslant 0$. 
Then, we have
\begin{align*}
L_{i}^{\prime}\left( x_{i}^{(1)} \right) < L_{i}^{\prime}\left( x_{i}^{(2)} \right).
\end{align*}
\end{assump}

After the FL model training, the model performance improvement of FL-PT $i\in\mathcal{C}$ is measured by:
\begin{align}\label{equa-model-improvement}
d_{i} = L_{i} - L_{i}^{\prime} \geqslant 0.
\end{align}
We focus on studying the effect of the final ML model performance improvements $\{d_{i}\}_{i=1}^{n}$ on the market shares. $\sum_{j=1}^{n}{d_{j}}$ denotes the aggregate model performance improvement for all FL-PTs. The relative model performance improvement $Q_{i}$ for FL-PT $i$ is defined as:
\begin{align}\label{equa-relative-model-quality}
Q_{i} = \frac{d_{i}}{\sum_{j=1}^{n}{d_{j}}},
\end{align}
where we have:
\begin{align}\label{equa-relative-model-quality-property}
Q_{i} \in [0, 1] \text{ and } \sum\nolimits_{i=1}^{n}{Q_{i}} = 1.
\end{align}
Given the value of $\sum_{j=1}^{n}{d_{j}}$, $\{Q_{i}\}_{i=1}^{n}$ corresponds to $\{d_{i}\}_{i=1}^{n}$. The service quality improvement $S_{i}^{\prime}$ of FL-PT $i$ is proportional to the model performance improvement $d_{i}$ and is defined as:
\begin{align}\label{equa-def-service-quality}
S_{i}^{\prime} = \alpha\cdot d_{i}
\end{align}
where, in a given market, the weight $\alpha>0$ is the degree of improvement to service quality brought by a unit of model performance improvement. The relative service quality improvement for $i$, $S_{i}$, is defined as:
\begin{align}\label{equa-def-relative-service-quality}
S_{i} = \frac{S_{i}^{\prime}}{\sum_{j=1}^{n}{S_{j}^{\prime}}} \overset{(a)}{=} Q_{i} \in [0, 1]
\end{align}
where the equality (a) is a result of Eq. \eqref{equa-def-service-quality}. We focus on a competitive market in which each firm serves a group of customers with no overlap. The attractiveness of firms to customers changes based on their relative service quality. The customers of a firm may switch to another firm. If the model performance of the firms improves in general, the entire market becomes more attractive, resulting in increases in the market size. Such market dynamics are detailed in the next subsection. We use $MS_{i}^{\prime}$ to denote the market share of $i$ after the FL process.

\subsection{FL-PT Market Share Dynamics}
\label{sec.characterize-market-shares}

\begin{figure}[t!]
\centering
\includegraphics[width=0.75\columnwidth]{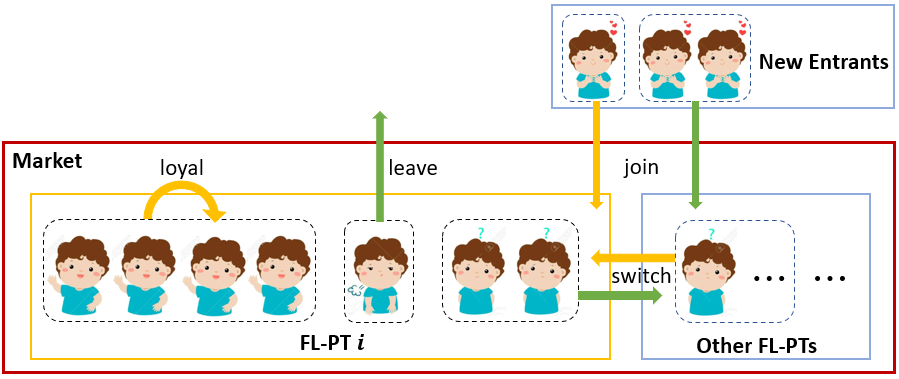}
\caption{The possible movements of customers among FL-PTs, and into and out of the market.}
\label{market-status}
\end{figure}

In this subsection, we adapt the classic model of Rust and Zahorik \cite{Rust93a} to characterize the FL-PTs' market shares $\{MS_{i}^{\prime}\}_{i=1}^{n}$ after FL training concludes. The market dynamics are mainly driven by (\rmnum{1}) the loyalty and the leaving or switching actions of the existing customers of different FL-PTs and (\rmnum{2}) the action of the customers who newly enter the market, which is illustrated in Figure~\ref{market-status}. Initially, the customers for $i\in\mathcal{C}$ (whose number totals to $P_{i}=MS_{i}\times P$) can be classified into three types:
\begin{itemize}
    \item Let $r_{i}\in [0,1]$ denote the proportion of customers loyal to $i$. After FL model training, $(1-r_{i})P_{i}$ customers will leave $i$. They will either exit the market or switch to other FL-PTs.
    \item Let $\nu_{i}\in [0,1-r_{i}]$ denote the proportion of customers who leave the market where we have
    \begin{align}\label{equa-range-r-nu}
    r_{i}+\nu_{i} \in [0,1].
    \end{align}
    After FL model training, $\nu_{i} P_{i}$ customers of $i$ will exit the market.
    \item The remaining $(1-r_{i}-\nu_{i}) P_{i}$ customers are still active in the market, but will consider the possibility of switching to other FL-PTs. These customers are referred to as ``free customers'' from $i$.
\end{itemize}
 Up to $\theta P$ new customers may join the market and be served by the $n$ firms, where $\theta\geqslant 0$. Generally, the market share of FL-PT $i\in\mathcal{C}$ consists of the following:
\begin{description}
\item [Loyal customers of $i$.] The number of loyal customers of $i$ is:
\begin{align}\label{equa-num-loyal-i}
P_{l,i} = r_{i} P_{i}.
\end{align}

\vspace{0.2em}\item [Free customers from other FL-PTs joining $i$.] We use $S_{i}$ defined in Eq. \eqref{equa-def-relative-service-quality} to denote the attractiveness of $i$ to customers. The fraction of free customers from $j\in\mathcal{C}$ joining $i$ is linearly proportional to $S_{i}$. Specifically, the number of customers leaving $j$ for $i$ equals to $(1-r_{j}-\nu_{j})P_{j}S_{i}$. The total number of free customers from itself and the other $(n-1)$ FL-PTs joining $i$ is:
\begin{align}\label{equa-num-free-to-i}
P_{f,i} = \sum\nolimits_{j\in \mathcal{C}}{(1-r_{j}-\nu_{j}) P_{j} S_{i}}.
\end{align}

\vspace{0.1em}\item [New customers joining $i$.] $\theta P$ new customers will be divided among the $n$ FL-PTs. The number of new customers joining $i$ is proportional to its attractiveness:
\begin{align}\label{equa-num-new-i}
P_{e,i} = S_{i} \theta P.
\end{align}
\end{description}
After FL training, $\sum\nolimits_{i=1}^{n}{\nu_{i}P_{i}}$ customers will exit the market; we assume that this amount is small such that the whole market size is still positive and the new market size becomes:
\begin{align}\label{equa-new-size}
P^{\prime} = (1+\theta)P-\sum\nolimits_{i=1}^{n}{\nu_{i}P_{i}} > 0.
\end{align}
Based on Eq. \eqref{equa-customer-number-i} and Eq. \eqref{equa-num-loyal-i}--\eqref{equa-new-size}, the market share of FL-PT $i$ becomes:
\begin{equation}\label{equa-changed-market-share}
\begin{split}
MS_{i}^{\prime} & =  \frac{P_{l,i} + P_{f,i} + P_{e,i}}{P^{\prime}}  = \frac{r_{i}MS_{i} + S_{i}\sum\nolimits_{j\in \mathcal{C}}{(1-r_{j}-\nu_{j})MS_{j}} + S_{i} \theta}{(1+\theta)-\sum_{j=1}^{n}{\nu_{j}MS_{j}}}.
\end{split}
\end{equation}

\subsection{Decision Support Tasks}
\label{sec.Decision-support-tasks}

Below, we formulate the decision support tasks to be addressed while understanding the role of FL in a competitive market.


\vspace{0.4em}\noindent\textbf{A Non-Cooperative Game.} The model of non-cooperative game provides insight into the interaction of FL-PTs. If a game has a dominant strategy equilibrium, such an equilibrium will be an ideal notion to predict the best course of action by any given player \cite{Osborne04a}.

The $n$ FL-PTs are self-interested and compete against each other. Each FL-PT $i$ has a decision variable $x_{i}$ which determine how many local data samples are used for FL training. $x_{i}$, in turn, determines the loss function value $L_{i}^{\prime}$ for $i$ after the FL model training concludes. This corresponds to a non-cooperative game where the $n$ FL-PTs are players. $i$'s strategy space is $\mathcal{X}_{i}=\left[0, D^{i}\right]$ and a single strategy is $x_{i}\in \mathcal{X}_{i}$. The strategy space of the other $(n-1)$ FL-PTs is denoted as $\mathcal{X}_{-i}=\left\{(x_{1}, \cdots, x_{i-1}, x_{i+1}, \cdots, x_{n}) \,|\, x_{i^{\prime}}\in \mathcal{X}_{i^{\prime}}, \forall i^{\prime}\in \mathcal{C}-\{i\}\right\}$.
$i$ aims to enhance its market status in a competitive market, and its payoff is defined as its market share $MS_{i}^{\prime}$ (or equivalently $MS_{i}^{\prime}P^{\prime}$). An ideal situation is that each FL-PT $i\in\mathcal{C}$ can decide its best strategy $x_{i}^{\ast}\in \mathcal{X}_{i}$ to maximize its market share, even if it does not have any knowledge of the strategies of other FL-PTs. Formally, $x^{\ast}=\{x_{i}^{\ast}\}_{i=1}^{n}$ is a dominant strategy equilibrium of the game if we have that, for each FL-PT $i\in\mathcal{C}$ and any alternate strategy $x_{i}^{\prime}\in \mathcal{X}_{i}$ different than $x_{i}^{\ast}$,
\begin{align}\label{equa-dominant-strategy}
MS_{i}^{\prime}(x_{i}^{\ast}, x_{-i}) \geqslant MS_{i}^{\prime}(x_{i}^{\prime}, x_{-i}), \forall x_{-i}\in \mathcal{X}_{-i}.
\end{align}


\vspace{0.4em}\noindent\textbf{FL Viability and Market Friendliness.} Another decision support task of importance to the manager of an FL ecosystem is to understand the viability of FL in a competitive market. As discussed above, the market dynamics are parameterized by $\theta$, $\{r_{j}\}_{j=1}^{n}$ and $\{\nu_{j}\}_{j=1}^{n}$. The effect of FL on model performance is parameterized by $\{Q_{j}\}_{j=1}^{n}=\{S_{j}\}_{j=1}^{n}$. The resulting market shares $\{MS_{i}^{\prime}\}_{i=1}^{n}$ after the FL model training are functions of these parameters.

As formally analyzed in the "Market Stability" subsection, for every FL-PT $i\in \mathcal{C}$, there is a tight lower bound of the relative ML model performance improvement, which is denoted by $\hat{Q}_{i}^{\min}$, to maintain the market $\delta$-stability. Formally, the market is $\delta$-stable if and only if $Q_{i}\geqslant \hat{Q}_{i}^{\min}, \forall i\in \mathcal{C}$. Mathematically, it is possible that $\hat{Q}_{i}^{\min}<0$. This means that the market share reduction of FL-PT $i$ will not exceed $\delta$ even if $i$ does not obtain model performance improvement in the FL model training process. {However, in a real FL ecosystem, each FL-PT $i$ can achieve some level of model performance improvement, {\em i.e.,} $Q_{i}\geqslant 0$.} Thus, we let
\begin{align}\label{equa-Q-i-min}
Q_{i}^{\min} = \max\left\{\hat{Q}_{i}^{\min}, 0\right\} \geqslant 0, \forall i\in \mathcal{C}
\end{align}
and the market $\delta$-stability is achievable if and only if the following relation is satisfied:
\begin{align}\label{equa-abstract-bound}
    Q_{i}\geqslant Q_{i}^{\min}, \forall i\in \mathcal{C}.
\end{align}
We now define an index to measure the friendliness of market environments towards FL. Let
    \begin{align}\label{equa-allocation-decomposition}
      Q_{i} = Q_{i}^{\min} + y_{i},  \forall i\in \mathcal{C}.
    \end{align}

\begin{definition}\label{def-friendliness}
Suppose that we are given the bounds $\{Q_{i}^{\min}\}_{i=1}^{n}$ and the decision variables are $\{Q_{i}\}_{i=1}^{n}$. The level of market friendliness towards FL, $\kappa$, is defined as:
\begin{align}\label{equa-friendliness}
\kappa \delequal 1 - \sum\nolimits_{i=1}^{n}{Q_{i}^{\min}} \overset{(a)}{=} \sum\nolimits_{i=1}^{n}{y_{i}}
\end{align}
where the equality (a) is due to Eq. (\ref{equa-relative-model-quality-property}) and Eq. (\ref{equa-allocation-decomposition}).
\end{definition}

\begin{figure}[t]
\centering
\subfigure[$\kappa=1-\sum_{i=1}^{2}{Q_{i}^{\min}}=0.4$]{\includegraphics[width=0.28\columnwidth]{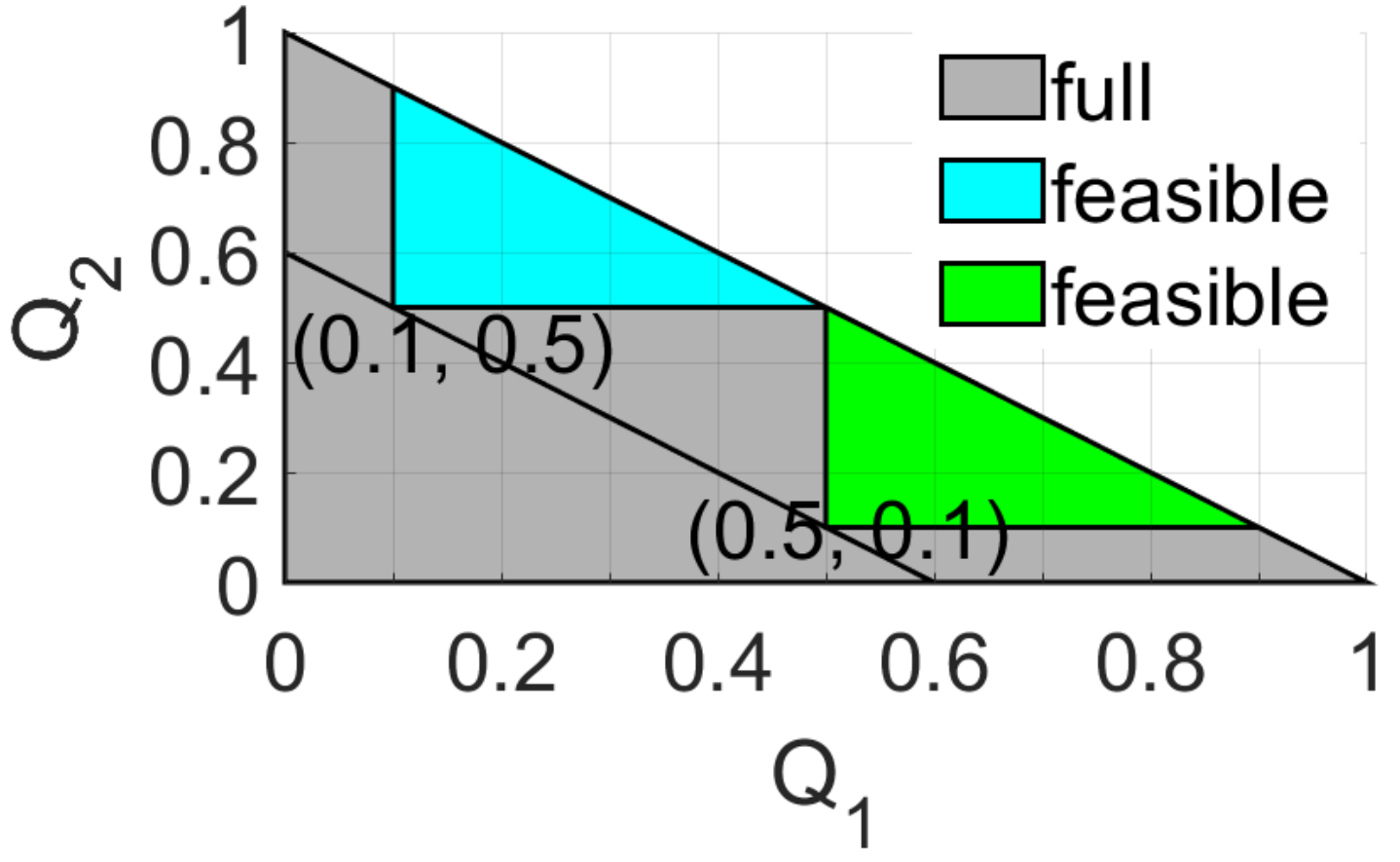}\label{Fig-Friendliness-Small}}
\hspace{2cm}\subfigure[$\kappa=1-\sum_{i=1}^{2}{Q_{i}^{\min}}=0.8$]{\includegraphics[width=0.28\columnwidth]{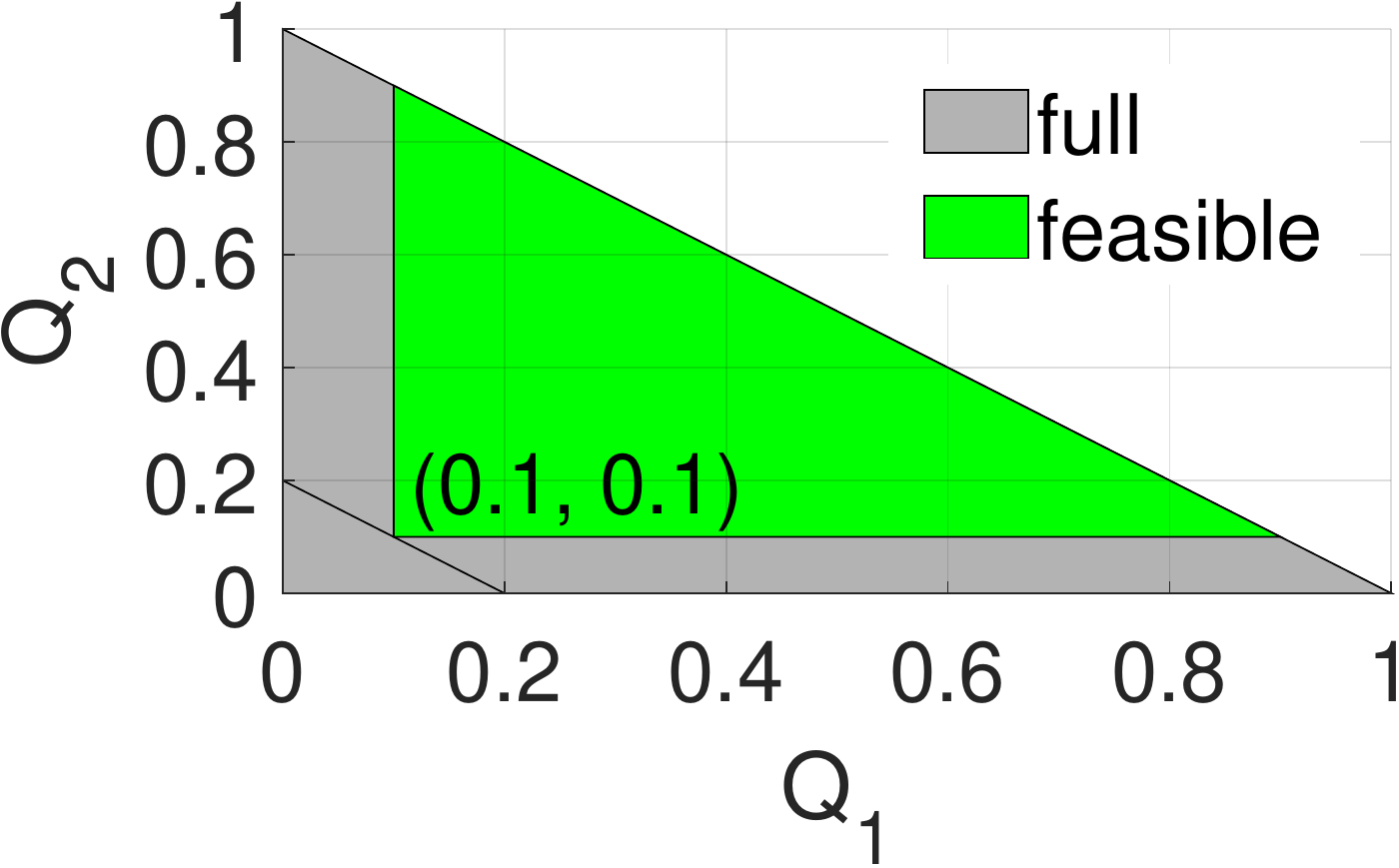}\label{Fig-Friendliness-Large}}
\caption{Friendliness of market environments to FL ($n=2$).}
\label{Fig-Friendliness}
\end{figure}

The value of $\kappa$ indicates the viability of FL. If $\kappa\in [0, 1]$, then the market $\delta$-stability can be maintained and FL is viable in a competitive market. If $\kappa\notin [0, 1]$, then FL is unviable. The explanation is as follows. If $\kappa\in [0, 1]$, then there exists a feasible allocation of $\{y_{i}\}_{i=1}^{n}$ among FL-PTs such that $y_{i} \geqslant 0$, $\forall i\in \mathcal{C}$; further, the relation (\ref{equa-abstract-bound}) is satisfied by Eq. (\ref{equa-allocation-decomposition}) and the market $\delta$-stability is maintained. Thus, FL is viable. The value of $\kappa$ determines the decision space size of an FL designer; the larger the value of $\kappa$, the more friendly a market is towards FL:
    \begin{itemize}
    \item When $\kappa=1$, then $\sum\nolimits_{i=1}^{n}{Q_{i}^{\min}} = 0$. Further, by (\ref{equa-Q-i-min}), we have $Q_{i}^{\min} = 0$, $\forall i\in\mathcal{C}$. This indicates that any FL framework can satisfy the market $\delta$-stability, since Eq. (\ref{equa-abstract-bound}) holds naturally.

    \item The case of $\kappa\in (0, 1)$ is illustrated in Figures~\ref{Fig-Friendliness}\subref{Fig-Friendliness-Small} and \ref{Fig-Friendliness}\subref{Fig-Friendliness-Large} where the grey triangle denotes the full region in which each point represents a possible pair $(Q_{1}, Q_{2})$. Each colored triangle (blue or green) denotes the feasible region {that is the decision space of an FL designer and} in which each point corresponds to a feasible pair $(Q_{1}, Q_{2})$ that can achieve the market $\delta$-stability. In Figure~\ref{Fig-Friendliness}\subref{Fig-Friendliness-Small}, we illustrate the cases with $\kappa=0.4$, but different $(Q_{1}^{\min}, Q_{2}^{\min})$ values. In Figure~\ref{Fig-Friendliness}\subref{Fig-Friendliness-Large}, we illustrate the case with a larger $\kappa=0.8$.

    \item When $\kappa=0$, then $\sum\nolimits_{i=1}^{n}{Q_{i}^{\min}}=1$. By Eq. (\ref{equa-relative-model-quality-property}), we have $\sum\nolimits_{i=1}^{n}{Q_{i}}=\sum\nolimits_{i=1}^{n}{Q_{i}^{\min}}$; the only feasible solution that can satisfy the relation (\ref{equa-abstract-bound}) is $Q_{i}=Q_{i}^{\min}$, $\forall i\in \mathcal{C}$. Then, the FL framework has to strictly satisfy this solution in order to achieve the market $\delta$-stability.
    \end{itemize}
In contrast, if $\kappa<0$, then $\sum\nolimits_{i=1}^{n}{y_{i}}<0$: there exists an FL-PT $i\in\mathcal{C}$ such that $y_{i}<0$. This indicates that the relation (\ref{equa-abstract-bound}) cannot be satisfied.
If $\kappa>1$, then $\sum\nolimits_{i=1}^{n}{Q_{i}^{\min}}<0$, which contradicts Eq. (\ref{equa-Q-i-min}).

In the next section, we will analyze the dominant strategy equilibrium of the game above, the bound $Q_{i}^{\min}$ for all $i\in\mathcal{C}$ and the friendliness $\kappa$.


\section{FL-PT Behaviour and Achievability of Market's $\delta$-Stability}


\subsection{A Dominant Strategy Equilibrium}


\begin{proposition}\label{proposi-dominant-strategy}
The dominant strategy equilibrium of the non-cooperative game is $x^{\ast}=(D^{1}, D^{2}, \cdots, D^{n})$.
\end{proposition}
\begin{proof}
It suffices to show that Eq. (\ref{equa-dominant-strategy}) holds when $x_{i}^{\ast}=D^{i}$. With abuse of notation, we let $Q_{1}$, $Q_{2}$, $\cdots$, $Q_{n}$ ({\em resp.} $Q_{1}^{\prime}$, $Q_{2}^{\prime}$, $\cdots$, $Q_{n}^{\prime}$) denote the relative model performance improvements of the $n$ FL-PTs under the strategy profile $\left(x_{i}^{\ast}, x_{-i}\right)$ ({\em resp.} $\left(x_{i}^{\prime}, x_{-i}\right)$). Based on Assumption~\ref{assump-more-data-contribution}, we have $L_{i}^{\prime}\left( x_{i}^{\ast} \right) < L_{i}^{\prime}\left( x_{i}^{\prime} \right)$. The loss function values of other FL-PTs remain unchanged. Thus, from Eq. \eqref{equa-model-improvement} and Eq. \eqref{equa-relative-model-quality}, we have:
\begin{align}
& Q_{i} > Q_{i}^{\prime} \text{ and } Q_{j} < Q_{j}^{\prime}, \forall j\in \mathcal{C}-\{i\}  \label{equa-Q-prime}
\end{align}
Let
\begin{equation}\label{ineq-2}
\begin{split}
 \varepsilon_{i} = Q_{i} - Q_{i}^{\prime} \overset{(a)}{>} 0.
\end{split}
\end{equation}
where the inequality (a) is due to Eq. \eqref{equa-Q-prime}. Finally,
\begin{equation*}
 MS_{i}^{\prime}(x_{i}^{\ast}, x_{-i}) - MS_{i}^{\prime}(x_{i}^{\prime}, x_{-i})
\overset{(c)}{=}  \frac{\sum\nolimits_{j\in \mathcal{C}}{(1-r_{j}-\nu_{j})MS_{j}\varepsilon_{i}} + \varepsilon_{i} \theta}{1+\theta-\sum\nolimits_{j=1}^{n}{\nu_{j}MS_{j}}} \overset{(d)}{\geqslant} 0
\end{equation*}
The equality (c) is due to Eq. \eqref{equa-def-relative-service-quality} and Eq. \eqref{equa-changed-market-share}; the inequality (d) is due to Eq. \eqref{equa-range-r-nu}, Eq. \eqref{equa-new-size} and Eq. \eqref{ineq-2}.
\end{proof}

In a competitive market, each FL-PT aims to maximize its market share. By Proposition~\ref{proposi-dominant-strategy}, an FL-PT's market share is maximized when it uses all of its local data for FL training, regardless of others' strategies. Ideally, under the dominant strategy solution, the FL ecosystem will also produce the best global ML model.

\subsection{Market Stability}
\label{sec.general-case}

We define the following factors to represent the overall market dynamics and show that the lower bound $Q_{i}^{\min}$ in (\ref{equa-abstract-bound}) is governed by some of these factors under a simple mathematical structure. Let
\begin{equation}\label{equa-overall-market-dynamics}
    \begin{split}
        v_{o} = \sum_{i=1}^{n}{v_{i}MS_{i}},\enskip e = 1+\theta-v_{o},\enskip f_{o} = \frac{\theta + \sum_{i=1}^{n}{(1-r_{j}-\nu_{i})MS_{i}}}{e},\enskip \text{and}\enskip \hat{r}_{i} = \frac{r_{i}MS_{i}}{e}
    \end{split}
\end{equation}
After the FL training process, we have
\begin{itemize}
\item $v_{o}$ represents the rate at which the customers leave the market (i.e., the total number of such customers is $v_{o}P$).

\item $(e \times P)$ represents the total number of customers in the market by Eq. (\ref{equa-new-size}). We refer to $e\in (0, +\infty)$ as \textbf{the expanded scale of the market} in relation to the original population size $P$.

\item $r_{i}MS_{i}P$ is the number of loyal customers of FL-PT $i$ in the original market. These customers will continue to be served by FL-PT $i$ after the FL training process. $\hat{r}_{i}$ is the ratio of $r_{i}MS_{i}P$ to $eP$. We refer to $\hat{r}_{i}$ as \textbf{the proportion of old customers of FL-PT $i$} in relation to the new population size $eP$.

\item $(1-r_{j}-\nu_{i})MS_{i}P$ is the number of free customers of FL-PT $i$ who will reconsider which FL-PT to join after FL model training. $\theta P$ is the number of new customers who consider joining one of the $n$ FL-PTs. Thus, $\theta + \sum_{i=1}^{n}{(1-r_{j}-\nu_{i})MS_{i}} \times P$ represents the total number of customers who eventually joins one FL-PT. We refer to $f_{o}$ as \textbf{the proportion of vacillating customers} in relation to the new population size $eP$.
\end{itemize}

Suppose we are given an arbitrary FL training framework in which each FL-PT $i\in \mathcal{C}$ uses $x_{i}$ local data samples for FL model training. This process determines the values of $\{Q_{1}, \dots, Q_{n}\}$ via Eq. \eqref{equa-relative-model-quality}. The proposition below shows that the $\delta$-stability of the market is achieved when $\{Q_{1}, \dots, Q_{n}\}$ satisfy a particular relation with the original market status and dynamics.

\begin{proposition}\label{proposi-general-result}
Define
\begin{align}
    \hat{Q}_{i}^{\min} = \frac{(MS_{i}-\delta) - \hat{r}_{i}}{f_{o}}.
\end{align}
The market stability is achieved if and only if $Q_{i} \geqslant Q_{i}^{\min}, \forall i\in\mathcal{C}$, where $Q_{i}^{\min}=\max\left\{\hat{Q}_{i}^{\min},\, 0\right\}$.
\end{proposition}
\begin{proof}
To achieve market stability, based on Eq. \eqref{equa-changed-market-share} and Definition~\ref{def-market-stability}, we have for each FL-PT $i\in \mathcal{C}$ that:
\begin{equation*}\label{equa-general-bound}
V_{i} = MS_{i} - MS_{i}^{\prime} = MS_{i} - \frac{r_{i}MS_{i} + S_{i}\sum\nolimits_{j\in \mathcal{C}}{(1-r_{j}-\nu_{j})MS_{j}} + S_{i} \theta}{(1+\theta)-\sum_{j=1}^{n}{\nu_{j}MS_{j}}} \leqslant \delta.
\end{equation*}
Further, we have
\begin{align*}
    S_{i} \geqslant \frac{(MS_{i}-\delta)\left( 1 + \theta - \sum_{j=1}^{n}{v_{j}MS_{j}} \right) - r_{i}MS_{i}}{\theta + \sum_{j=1}^{n}{(1-r_{j}-v_{j})MS_{j}}}
\end{align*}
By Eq. \eqref{equa-def-relative-service-quality}, we have $S_{i}=Q_{i}$. Further, based on Eq. \eqref{equa-abstract-bound} and Eq. \eqref{equa-overall-market-dynamics}, Proposition~\ref{proposi-general-result}.
\end{proof}

Proposition~\ref{proposi-general-result} provides the minimum relative ML model performance improvements $\left\{Q_{i}^{\min}\right\}_{i=1}^{n}$ required to maintain the market $\delta$-stability. The lower bound $Q_{i}^{\min}$ of FL-PT $i$ is governed by its own market features and the overall market dynamics, including the original market share $MS_{i}$, the proportion $\hat{r}_{i}$ of old customers of FL-PT $i$, and the proportion $f_{o}$ of the vacillating customers of all FL-PTs.
For FL-PTs with $\hat{Q}_{i}^{\min}\leqslant 0$, their $\delta$-stability can be maintained even if they do not benefit from the FL training process and their ML model performance is not improved. Conversely, we refer to an FL-PT $i\in\mathcal{C}$ with $\hat{Q}_{i}^{\min}>0$ as {\em a sensitive FL-PT}, and its ML model performance has to be improved to maintain the $\delta$-stability. The set of sensitive FL-PTs is denoted by $\mathcal{C}^{\prime}$ and we let $n^{\prime}=|\mathcal{C}^{\prime}|$. Next, by Definition~\ref{def-friendliness}, we can obtain the friendliness towards FL under any given market environment.


\begin{proposition}\label{Proposi-friendliness}
The friendliness of market environment to FL is:
\begin{align*}
\kappa = 1 - \sum\limits_{i=1}^{n}{\max\left\{ 0, \frac{(MS_{i}-\delta) - \hat{r}_{i}}{f_{o}} \right\}} = 1 - \frac{\sum\nolimits_{i\in\mathcal{C}^{\prime}}{\left(MS_{i} - \hat{r}_{i}\right)} - n^{\prime}\delta}{f_{o}}
\end{align*}
\end{proposition}
\begin{proof}
It follows from Definition~\ref{def-friendliness} and Proposition~\ref{proposi-general-result}.
\end{proof}

By Proposition \ref{Proposi-friendliness}, $\kappa$ is governed by the number $n^{\prime}$ of sensitive FL-PTs, the total market share $\sum\nolimits_{i\in\mathcal{C}^{\prime}}{MS_{i}}$ of these sensitive FL-PTs $\mathcal{C}^{\prime}$, the total proportion $\sum\nolimits_{i\in\mathcal{C}^{\prime}}{\hat{r}_{i}}$ of the old customers of $\mathcal{C}^{\prime}$, and the proportion $f_{o}$ of vacillating customers. It is observed that the friendliness of market environments to FL simply by the four factors that reflect the overall market dynamics, instead of the individual's market features.

\begin{proposition}\label{Proposi-viability}
FL is viable in a competitive market if and only if $\frac{\left(\sum\nolimits_{i\in\mathcal{C}^{\prime}}{MS_{i} - \hat{r}_{i}}\right) - f_{o}}{|\mathcal{C}^{\prime}|} \leqslant \delta$.
\end{proposition}
\begin{proof}
FL is viable in a competitive market if and only if $\kappa>0$. The proposition follows from Proposition \ref{Proposi-friendliness}.
\end{proof}


\section{Numerical Studies and Analysis}
In this section, we illustrate the minimum requirement of the relative ML model performance improvement to maintain the market $\delta$-stability, and the friendliness to FL in typical market environments.

\subsection{General Settings}
The number $n$ of FL-PTs can be arbitrary. Different types of markets have different growth rates. We set $\theta$ to 0.1 and 0.5 to simulate a mature market and a fast-growing market, respectively. The average proportion of loyal customers, $r_{o}$, is set to the range of $[0.7, 0.95]$ (i.e., we set $r_{1}=\cdots=r_{n}=r_{o}$ and for $i\in\mathcal{C}$, $r_{i}$ ranges from 0.7 to 0.95 with a stepsize 0.05). The rates at which customers leave an FL-PT and all the $n$ FL-PTs are generally small and set to 0.02 (i.e., $v_{1}=\cdots=v_{n}=v_{o}=0.02$).

\subsection{The Minimum Requirement $\{Q_{i}^{\min}\}_{i=1}^{n}$}

The market share represents the competitiveness of a firm $i\in\mathcal{C}$. We set $MS_{i}$ to the range of $[0.2, 0.8]$ with a stepsize 0.1. The proportion $\hat{r}_{i}$ of old customers and the proportion $f_{o}$ of vacillating customers are computed by Eq. (\ref{equa-overall-market-dynamics}). The minimum ML model performance improvement requirement of an individual FL-PT $i$, $Q_{i}^{\min}$, is illustrated in Figure \ref{Fig-Performance-Bound}. Overall, different FL-PTs have different requirements $Q_{i}^{\min}$ to maintain the market $\delta$-stability based on their market shares, their customer loyalties and the market dynamics such as the market growth rate:
\begin{itemize}
    \item In spite of the market growth rate, we can see from each plot of Figure \ref{Fig-Performance-Bound} that (\rmnum{1}) the higher the market share $MS_{i}$ of an FL-PT $i$, the larger the required value of $Q_{i}^{\min}$, and (\rmnum{2}) the higher the customer loyalty $r_{i}$ of an FL-PT, the smaller the required value of $Q_{i}^{\min}$.
    \item Under the same market share and customer loyalty $(MS_{i}, r_{i})$, the minimum requirement $Q_{i}^{\min}$ in a mature market is lower than its counterpart in a fast-growing market.
\end{itemize}
While designing an FL scheme for FL-PTs in a competitive market, it is important to ensure that model performance improvement achievable for each FL-PT satisfies the minimum requirement $\{Q_{i}^{\min}\}_{i=1}^{n}$.

\begin{figure}[ht]
\centering
\subfigure[Mature market]{\includegraphics[width=0.44\columnwidth]{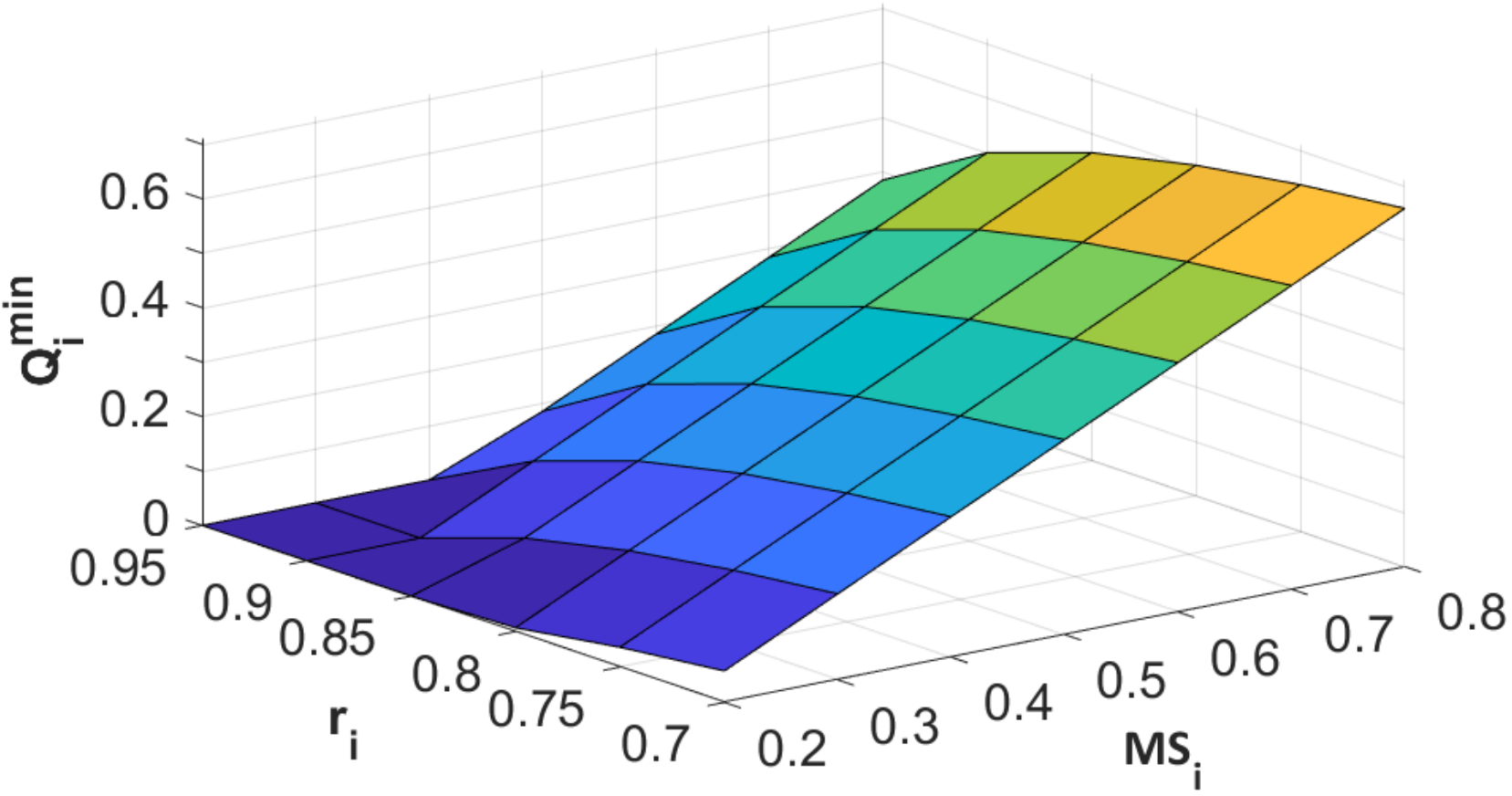}
\label{Fig-Comparable-Slow-Bound}}
\hspace{0.7cm}\subfigure[Fast-growing market]{\includegraphics[width=0.44\columnwidth]{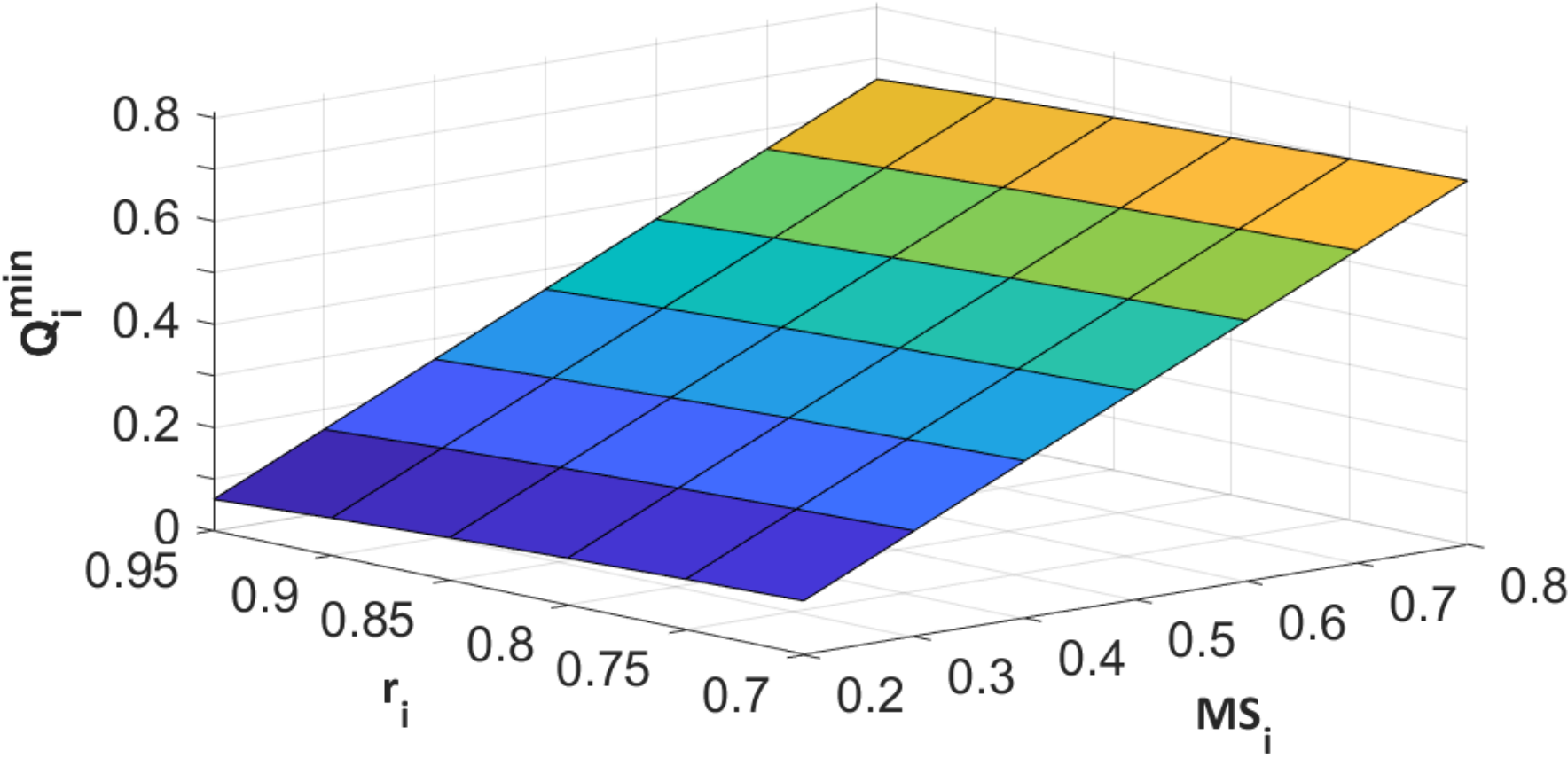}
\label{Fig-Comparable-Fast-Bound}}
\caption{The minimum relative improvement, $Q_{i}^{\min}$, required to maintain the market $\delta$-stability}
\label{Fig-Performance-Bound}
\end{figure}

\subsection{The Friendliness $\kappa$}

Overall, the friendliness, $\kappa$, in Proposition \ref{Proposi-friendliness} reflects how accommodating a competitive market is to FL and its meaning is explained when we define it in Definition~\ref{def-friendliness}. Let $MS_{o}^{\prime}=\sum\nolimits_{i\in\mathcal{C}^{\prime}}{MS_{i}}$. We set $MS_{o}^{\prime}$ to range from 0.2 to 0.8 with a stepsize 0.1. $r_{o}$ represents the proportion of the loyal customers of all FL-PTs in the original market. Figure~\ref{Fig-Friendliness} illustrates the numerical results in a mature and fast-growing market where $(MS_{o}^{\prime}, r_{o})$ satisfies $\sum\nolimits_{i\in\mathcal{C}^{\prime}}{\left(MS_{i} - \hat{r}_{i}\right)} - n^{\prime}\delta = MS_{o}^{\prime}\left(1-\frac{r_{o}}{e} \right) > 0$. Given the overall customer leaving rate $v_{o}$, a high market growth rate $\theta$ leads to a high expanded scale $e$ of the market by Eq. (\ref{equa-overall-market-dynamics}). By Proposition \ref{Proposi-friendliness}, the friendliness $\kappa$ is increasing in the overall customer loyalty $r_{o}$ and the number $n^{\prime}$ of sensitive FL-PTs, and decreasing in the total market share $MS_{o}^{\prime}$ of sensitive FL-PTs and the market growth rate $\theta$, which are also reflected in the six plots of the figure. Encouragingly, $\kappa \in (0, 1)$ in a wide range of market environments; so is the FL viability:
\begin{itemize}
    \item In spite of the market growth rate $\theta$ and the number $n^{\prime}$ of sensitive FL-PTs, it is observed from each plot of Figure \ref{Fig-Friendliness} that $\kappa \in (0, 1)$ under all the feasible pairs $(MS_{o}^{\prime}, r_{o})$.

    \item The friendliness $\kappa$ in a mature market is higher, when the number $n^{\prime}$ of sensitive FL-PTs, the total market share $MS_{o}^{\prime}$ of sensitive FL-PTs, and the overall customer loyalty $r_{o}$ are the same.

    \item From each of Figures \ref{Fig-Friendliness-Slow-1}-\ref{Friendliness-Slow-3}, it is observed in a mature market that the friendliness $\kappa$ is especially high when either the total market share $MS_{o}^{\prime}$ of sensitive FL-PTs is small or the overall customer loyalty $r_{o}$ is high.

    \item From each of Figures \ref{Fig-Friendliness-Fast-1}-\ref{Fig-Friendliness-Fast-3}, it is observed in a fast-growing market that (\rmnum{1}) if the number $n^{\prime}$ of sensitive FL-PTs is small (e.g., $n^{\prime}=1, 2$), the friendliness $\kappa$ is especially high when the total market share $MS_{o}^{\prime}$ of sensitive FL-PTs is small, in spite of the overall customer loyalty $r_{o}$, and (\rmnum{2}) if $n^{\prime}$ is large (e.g, $n^{\prime}=3$), we have the same observation as the mature market above.
\end{itemize}

\begin{figure}[t!]
\centering
\subfigure[Mature market, $n^{\prime}=1$]{\includegraphics[width=0.43\columnwidth]{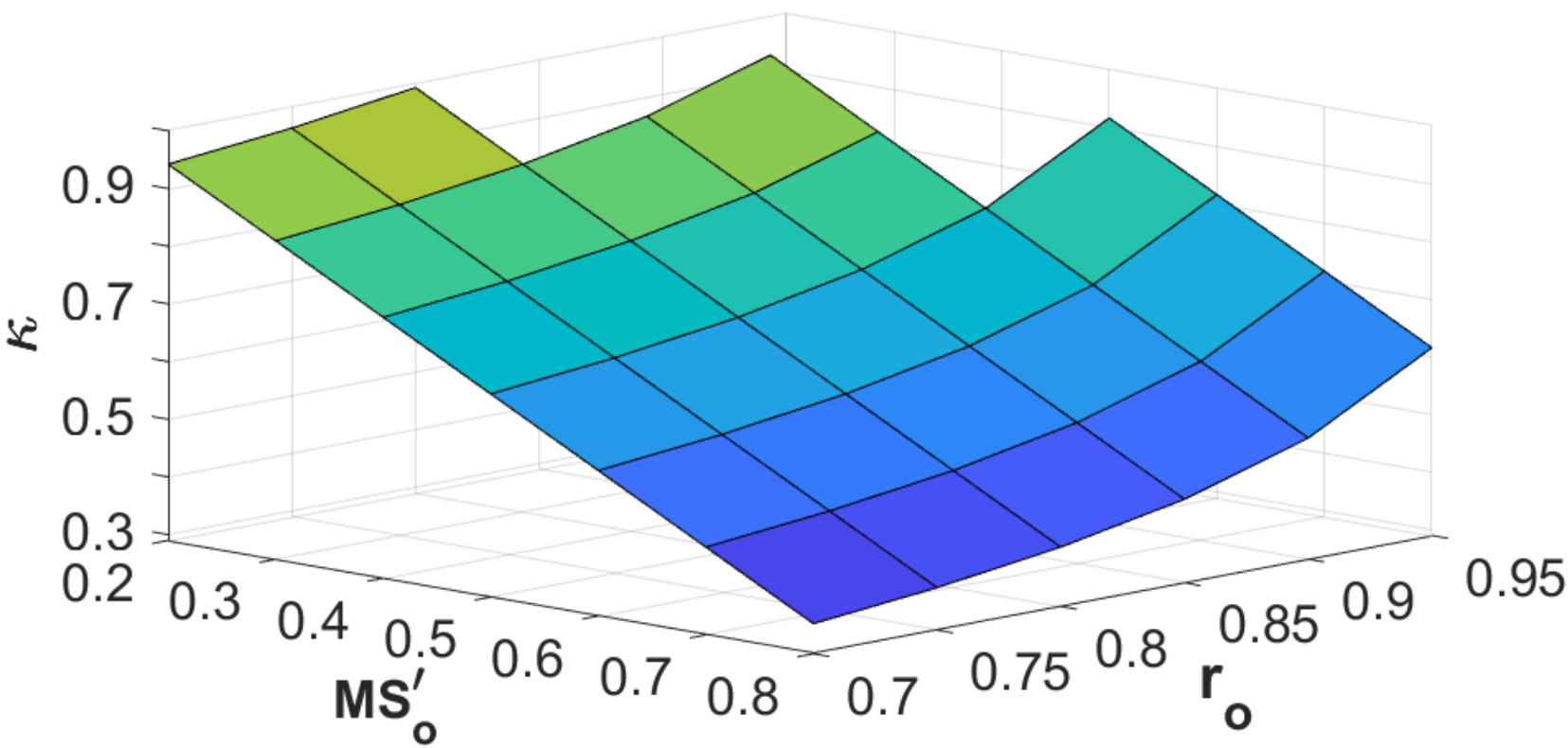}
\label{Fig-Friendliness-Slow-1}}
\hspace{0.7cm}\subfigure[Mature market, $n^{\prime}=2$]{\includegraphics[width=0.43\columnwidth]{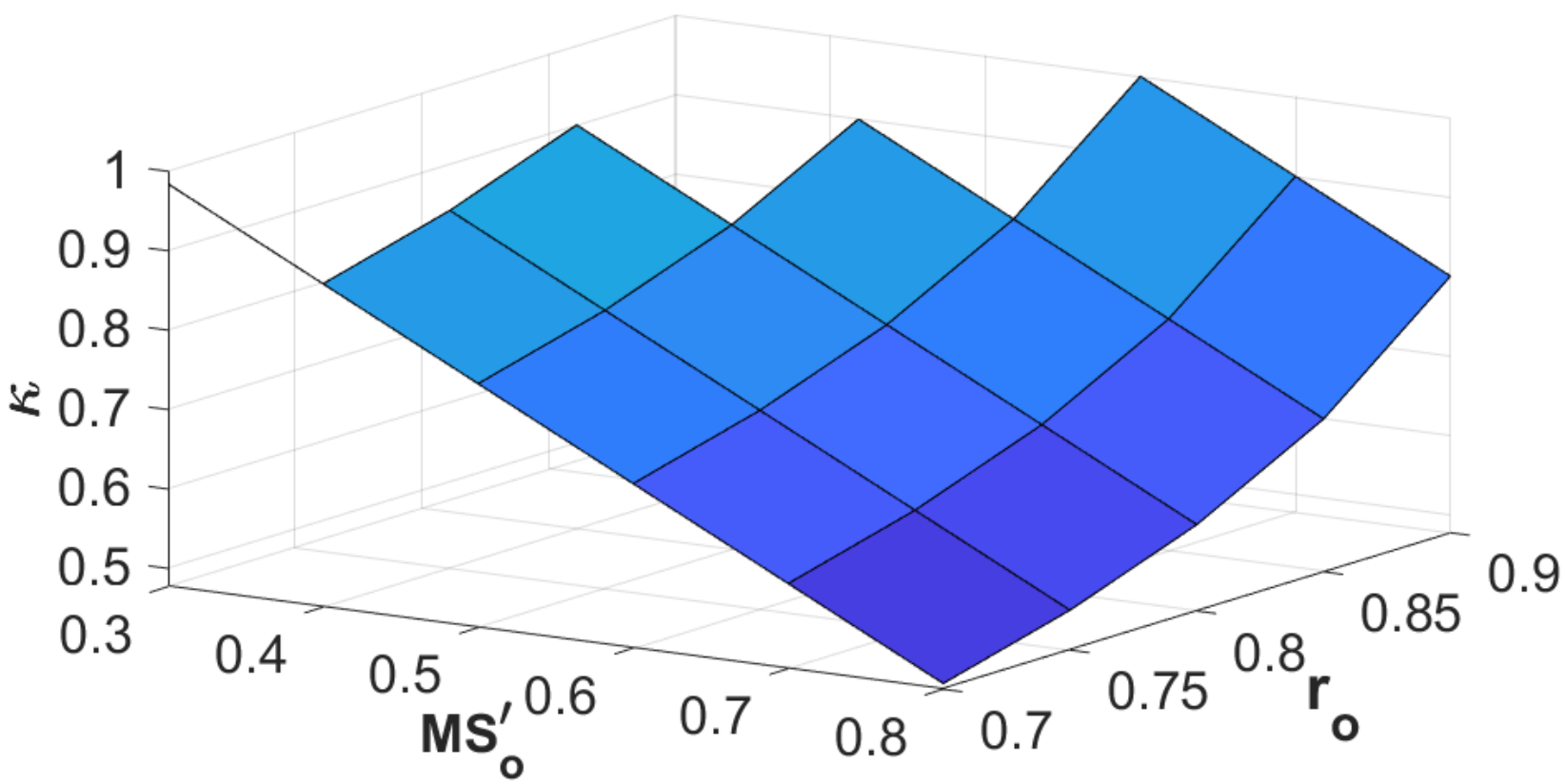}\label{Friendliness-Slow-2}}\\
\subfigure[Mature market, $n^{\prime}=3$]{\includegraphics[width=0.43\columnwidth]{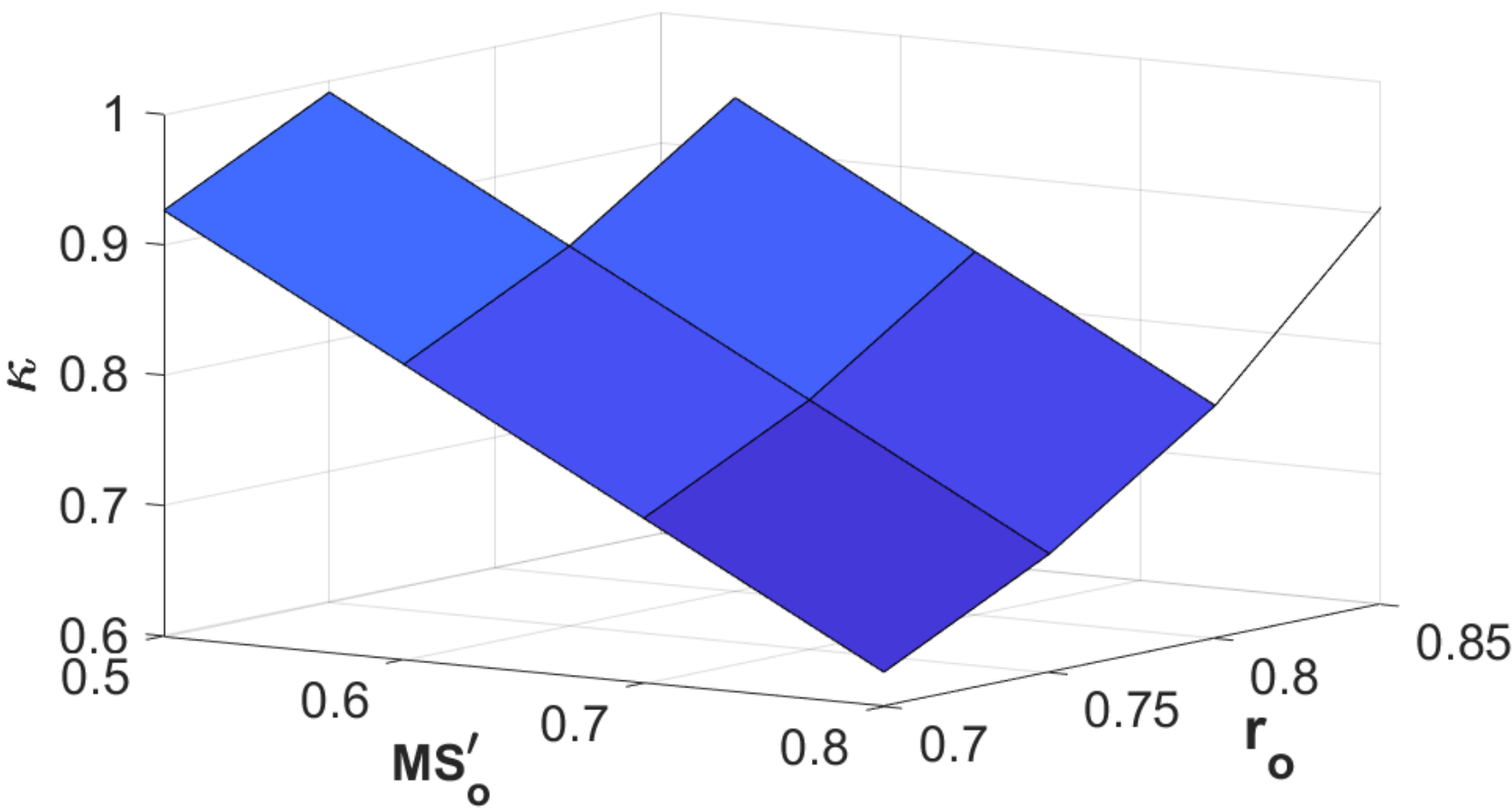}\label{Friendliness-Slow-3}}
\hspace{0.7cm}\subfigure[Fast-growing market, $n^{\prime}=1$]{\includegraphics[width=0.43\columnwidth]{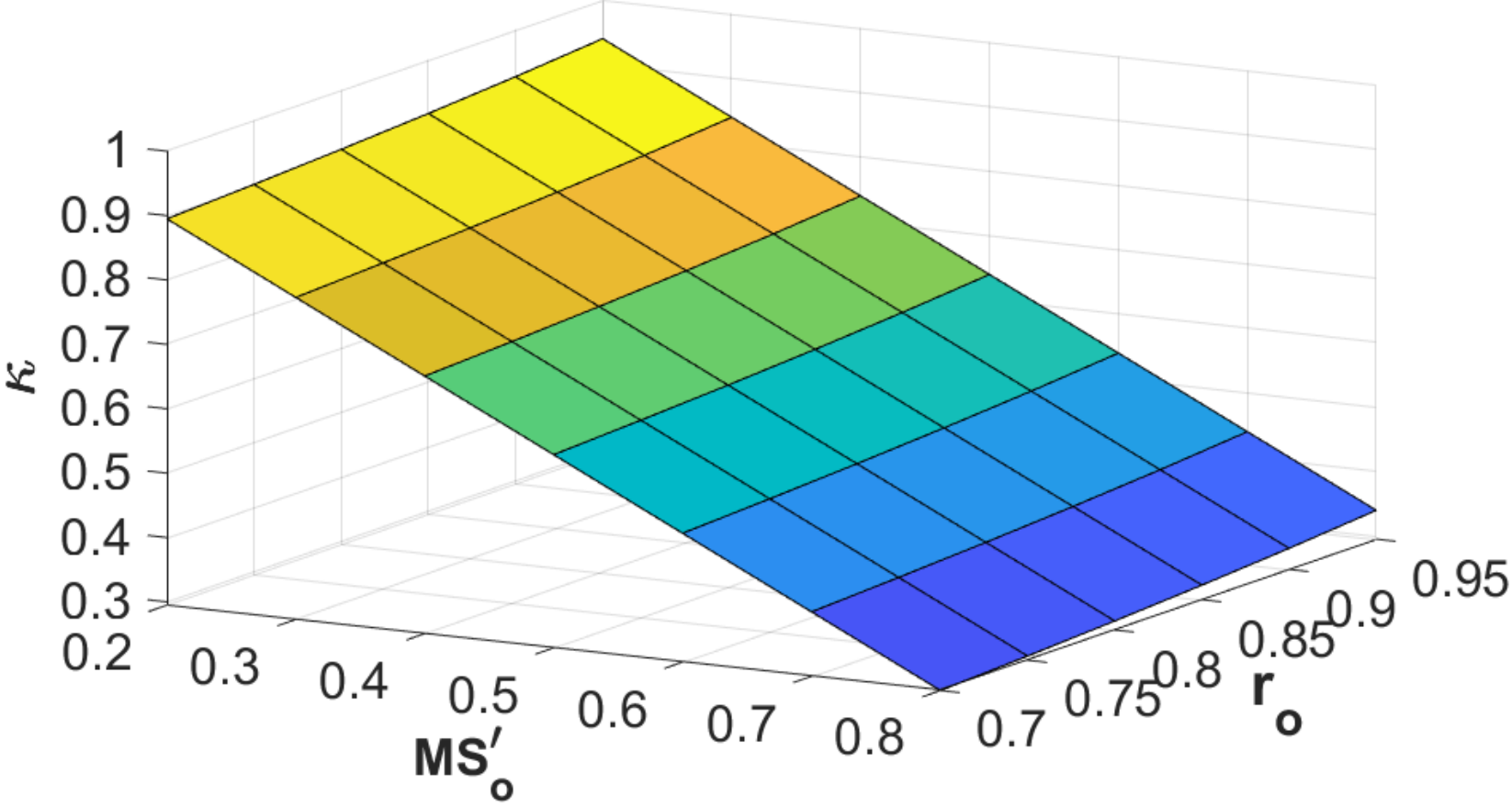}
\label{Fig-Friendliness-Fast-1}}\\
\subfigure[Fast-growing market, $n^{\prime}=2$]{\includegraphics[width=0.43\columnwidth]{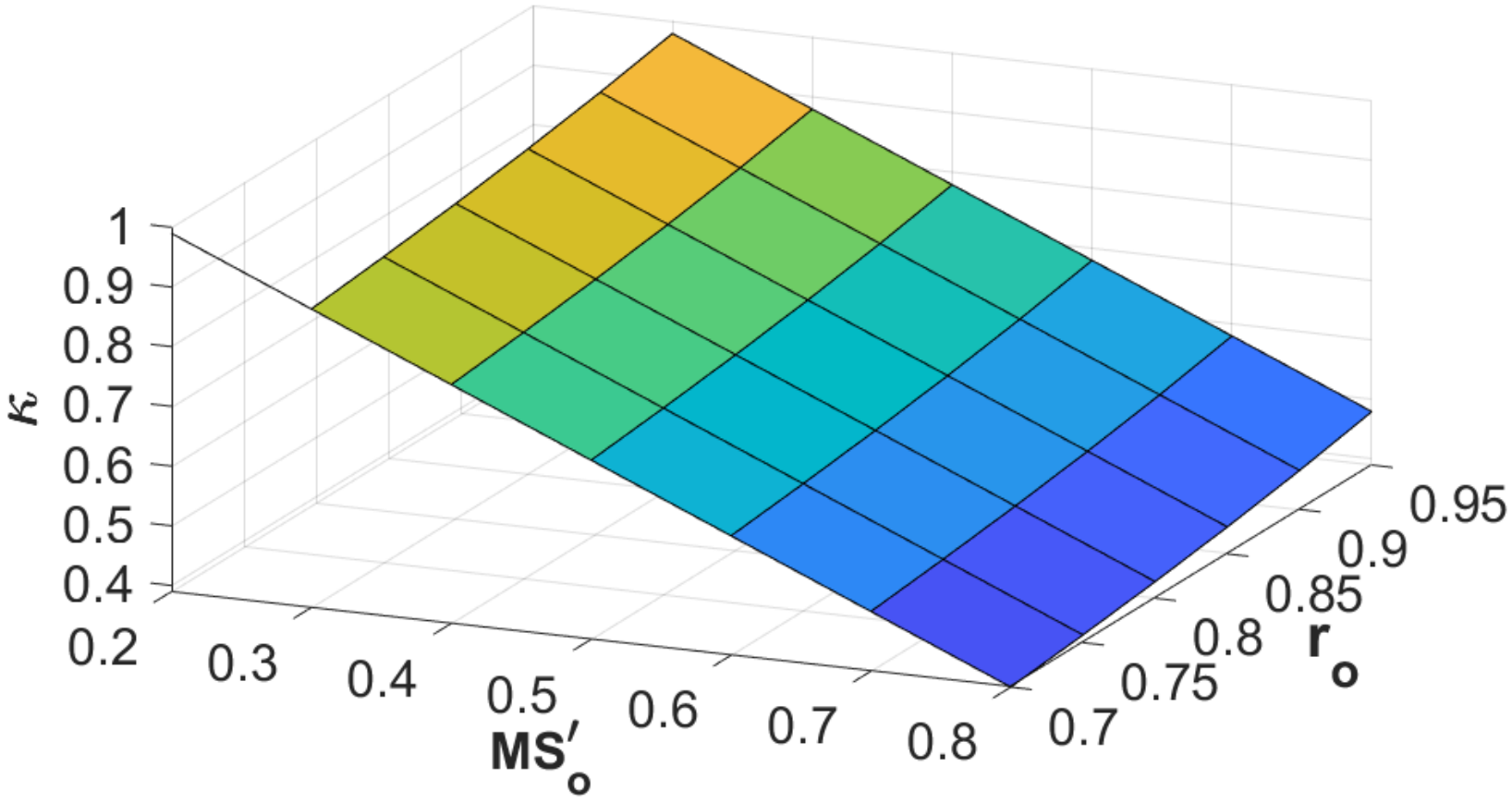}
\label{Fig-Friendliness-Fast-2}}
\hspace{0.7cm}\subfigure[Fast-growing market, $n^{\prime}=3$]{\includegraphics[width=0.43\columnwidth]{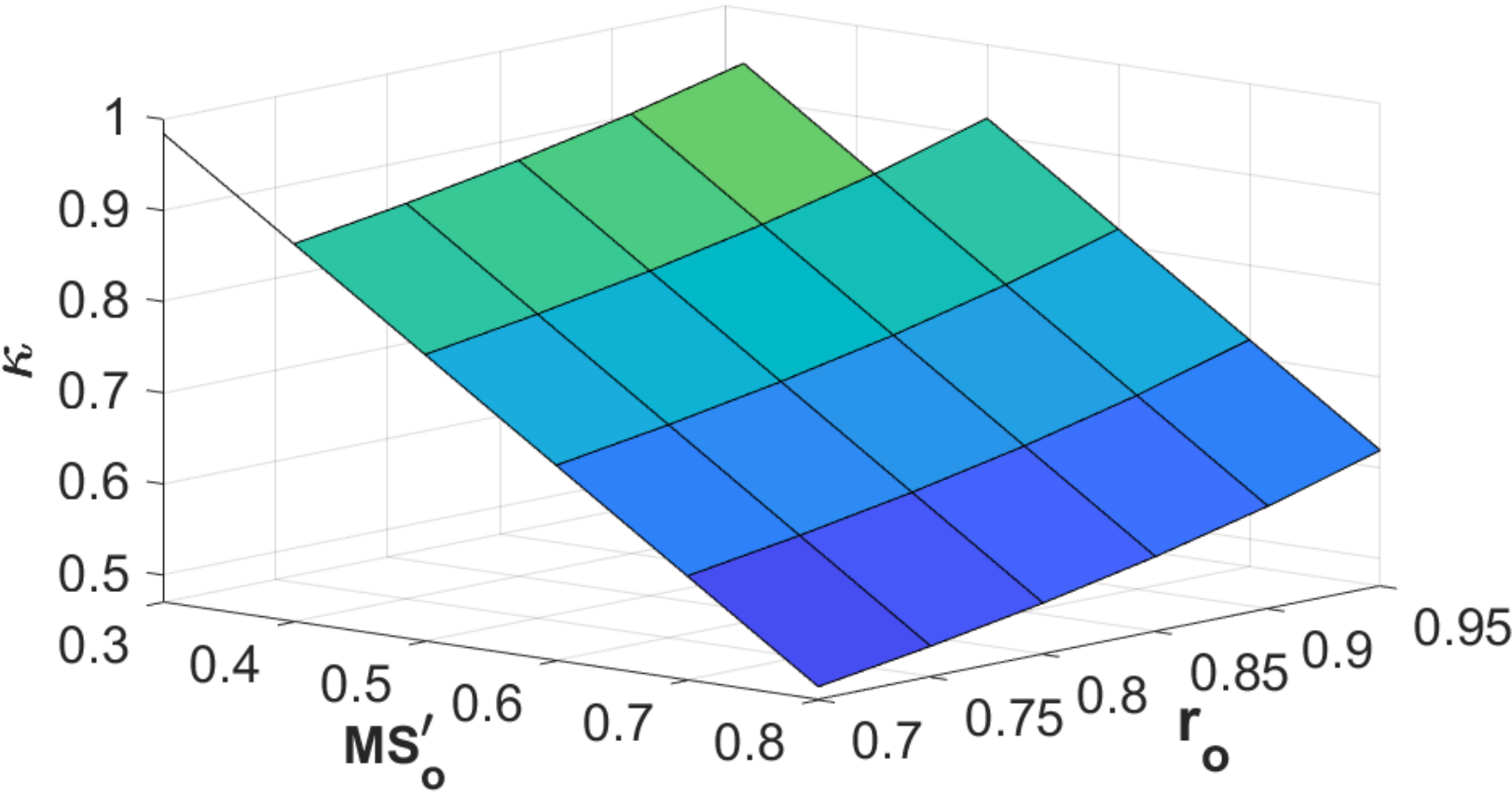}
\label{Fig-Friendliness-Fast-3}}
\caption{The Friendliness $\kappa$}
\label{Fig-Friendliness}
\end{figure}

\section{Conclusions and Future Work}

In this paper, we propose an analytical framework to understand the impact of FL on firms' market shares under various market settings.  For each FL-PT, we characterize the process by which it joins FL as a non-cooperative game and derive its dominant strategy. For the decision-makers of an FL ecosystem, MarS-FL provides a tight lower bound of the minimum relative ML model performance improvement required by each FL-PT in order to maintain the market $\delta$-stability. MarS-FL also provides a notion of {\em friendliness} to measure how conducive a market is for FL, and a sufficient and necessary condition for the viability of FL in a competitive market. The results of this paper can guide non-monetary FL incentive mechanisms to allocate model performance improvements among FL-PTs in order to encourage larger data owners to overcome their fear of smaller FL-PTs free-riding on them and join FL. 

In future, we will consider designing under the centralized architecture a generic parametric algorithm to determine the allocation of ML model performance while maintaining the market $\delta$-stability. Furthermore, in practice, the related market parameters such as customer loyalty and switching are sometimes estimated with uncertainty \cite{verbeke11a}. We will leverage the theory of robust optimization\cite{bertsimas2011theory} to address this. Overall, this paper provides a conceptual framework that serves as a starting point and allows for the further development of FL in the scenarios whenever FL-PTs are in a competitive market \cite{Zhan21a,yang2021toward,kalloori2021horizontal}.

\bibliographystyle{ACM-Reference-Format}
\bibliography{sample}

\end{document}